\documentclass[11pt,a4paper]{article}
\usepackage[hyperref]{acl2020}
\usepackage{times}
\usepackage{latexsym}
\usepackage{amsmath}
\usepackage{amssymb}
\usepackage{color}
\usepackage{booktabs}
\usepackage{multirow}
\usepackage{mathtools}
\usepackage{amsthm}
\usepackage{appendix}

\newtheorem{theorem}{Theorem}
\newtheorem{lemma}{Lemma}
\newtheorem{proposition}{Proposition}

\newcommand{\norm}[1]{\left\lVert#1\right\rVert}

\newcommand{\cB}{y_B}

\newcommand{\grs}{{g^\rs}}
\newcommand{\hatgrs}{{\hat{g}^\rs}}
\newcommand{\grsstar}{{{g}^\rs_*}}
\usepackage{microtype}

\aclfinalcopy 
\def\aclpaperid{412} 

\newcommand{\cutforspace}[1]{}

\title{SAFER: A Structure-free Approach for Certified Robustness to Adversarial Word Substitutions}

\author{Mao Ye\thanks{~~Equal contribution} \\
  UT Austin\\
  \texttt{my21@cs.utexas.edu} \\\And
  Chengyue Gong$^*$ \\
  UT Austin\\
  \texttt{cygong@cs.utexas.edu} \\\And
  Qiang Liu \\
  UT Austin\\
  \texttt{lqiang@cs.utexas.edu} \\
  }

\date{}

\begin{document}
\global\long\def\X{\textbf{X}}%
\global\long\def\x{{x}}%
\global\long\def\H{\text{H}}%
\global\long\def\h{\text{h}}%
\global\long\def\supp{\mathcal{X}}%
\global\long\def\Z{\textbf{Z}}%
\global\long\def\z{\textbf{z}}%
\global\long\def\rs{\text{RS}}%
\maketitle
\begin{abstract}
State-of-the-art NLP models can often be fooled by human-unaware transformations such as synonymous word substitution. For security reasons, it is of critical importance to develop models  with \emph{certified robustness} that can provably guarantee that 
the prediction is can not be altered by any possible synonymous word substitution.  
In this work, we propose a certified robust method based on a new randomized smoothing technique, which constructs a stochastic ensemble by applying random word substitutions on the input sentences, and leverage the statistical properties of the ensemble to provably certify the robustness. Our method is simple and \emph{structure-free} in that it only requires the black-box queries of the model outputs, and hence can be applied to any pre-trained models (such as BERT) and any types of models (world-level or subword-level). Our method significantly outperforms recent state-of-the-art methods for certified robustness on both IMDB and Amazon text classification tasks. 
To the best of our knowledge, we are the first work to achieve certified robustness on large systems such as BERT with practically meaningful certified accuracy. 

\end{abstract}

\section{Introduction}

Deep neural networks have achieved state-of-the-art results in many NLP tasks,
but also have been shown to be brittle to carefully crafted adversarial perturbations, such as 
replacing words with similar words ~\citep{alzantot2018generating}, 
adding extra text ~\citep{wallace2019universal}, 
and replacing sentences with semantically similar sentences ~\citep{ribeiro2018semantically}. 
These adversarial perturbations  are 
imperceptible to humans, 
but can fool deep neural networks and break their performance.  
 Efficient methods for defending these attacks 
 are of critical importance for deploying modern deep NLP models to practical automatic AI systems. 

In this paper, we focus on defending the synonymous 
word substitution attacking ~\citep{alzantot2018generating}, 
in which an attacker attempts to alter the output of the model by replacing words in the input sentence with their synonyms according to 
a synonym table, while keeping the meaning of this sentence unchanged. 
A model is said to be \emph{certified robust} if such an attack is guaranteed to  fail, no matter how the attacker manipulates the input sentences. 
 Achieving and verifying certified robustness is highly challenging even if the synonym table used by the attacker is known during training 
 \citep[see][]{jia2019certified}, 
 because it requires 
 to check every possible 
 synonymous word substitution, whose number is exponentially large. 

Various defense methods against 
 synonymous word substitution attacks have been developed 
 ~\citep[e.g.,][]{wallace2019universal, ebrahimi2017hotflip}, most of which, however, 
are 
not {certified robust} in that they may eventually be broken by stronger attackers. 
Recently, 
\citet{jia2019certified, huang2019achieving} proposed 
the first certified robust methods against word substitution attacking. Their methods are based on the interval bound propagation (IBP)  method 
\citep{dvijotham2018ibp} which computes the range of the model output by propagating the interval constraints of the inputs layer by layer. 


However, the IBP-based methods of \citet{jia2019certified, huang2019achieving} 
are limited in several ways. 
First, 
because IBP only works for certifying neural networks with continuous inputs, the inputs in \citet{jia2019certified} and \citet{huang2019achieving} are taken to be the word embedding  vectors of the input sentences, instead of the discrete sentences.
This makes 
it inapplicable to character-level~\citep{zhang2015charnn} and subword-level~\citep{bojanowski2017subword} model, which are more widely used in practice \citep{wu2016google}. 
\cutforspace{ 
Secondly, 
IBP requires to access the structures of the models, 
and needs to be derived and implemented at a model-by-model basis, 
which makes it difficult to use complex  models,  especially large-scale pre-trained models such as
BERT ~\citep{devlin2018bert}. 
Finally, 
the embedding vectors in NLP models can be non-local and hence make the IBP bounds relatively loose, 
because semantically similar words can have very different embedding vectors~\citep{mu2018allbutthetop}. 
In comparison, the adversarial perturbations in image classifications are constrained in small balls (e.g., $\ell_2$ or $\ell_1$), on which IBP can provide tighter bounds. 
}

In this paper, 
we 
propose a \emph{structure-free}  
certified defense method 
that applies to arbitrary models that can be queried in a black-box fashion, 
without any requirement on the model structures. 
Our method is based on the idea of  randomized smoothing, 
which smooths
 the model with random word substitutions 
 build on the synonymous network, 
 and leverage the statistical properties of the randomized ensembles to construct provably certification bounds. 
Similar ideas of provably certification 
using  
randomized smoothing 
have been developed recently in deep learning \citep[e.g.,][]{cohen2019certified, salman2019provably, zhang2020black, lee2019tight}, 
but mainly 
for computer vision tasks whose inputs (images) are in a continuous space \citep{cohen2019certified}. 
Our method admits a substantial extension  
of the randomized smoothing technique to discrete and structured input spaces for NLP. 

We test our method on various types of NLP models, including text CNN \citep{kim2014textcnn}, Char-CNN \citep{zhang2015charnn}, and BERT \citep{devlin2018bert}. 
Our method significantly outperforms the recent IBP-based methods \citep{jia2019certified, huang2019achieving}  
on both IMDB and Amazon text classification. 
In particular, we achieve an 87.35\% certified accuracy on IMDB by applying our method on the state-of-the-art BERT, 
on which previous certified robust methods are not applicable. 

\section{Adversarial Word Substitution}  \label{sec:background}
In a text classification task, 
a model $f(\X)$  maps an   input sentence $\X\in\mathcal{X}$ to a label $c$ in a  set $\mathcal{Y}$ of discrete categories, where
$\X = \x_{1},\ldots, \x_{L}$ is a sentence consisting of $L$ words. 
In this paper, we focus on adversarial word substitution in which
an  attacker 
arbitrarily replaces the words in the sentence by their synonyms according to a synonym table to alert the prediction of the model. 
Specifically, for any word $\x$, we consider a pre-defined synonym set $S_{\x}$ that contains the synonyms of $\x$ (including $\x$ itself). 
We assume the synonymous relation is symmetric, that is, 
$\x$ is in the synonym set of all its synonyms. 
The synonym set $S_{\x}$ can be built based on GLOVE~\citep{pennington2014glove}.

With a given input sentence $\X = \x_{1}$,\ldots, $\x_{L}$, 
the attacker may construct an adversarial sentence ${\X'}=\x'_{1},\ldots,\x'_{L}$  by perturbing at most 
$R\le L$ words $x_i$ in $\X$ to any of their synonyms $ x'_i \in S_{x_i}$,
%
\begin{align*}
\begin{split} 
 & 
 S_{\X} :=  \left \{ 
%
  {\X'}\colon~
  \norm{\X' - \X}_{0} \leq R, ~~\x'_{i}\in S_{\x_{i}}, \forall i
\right\},
 \end{split}
\end{align*}
where $S_{\X}$ denotes the candidate set of adversarial sentences available to the attacker.  
Here $\norm{ \X' - \X}_{0} := \sum_{i=1}^{L}\mathbb{I}\left\{ \x'_{i}\ne\x_{i}\right\}$ is the Hamming distance, with  $\mathbb{I}\{\cdot\}$ the indicator function. 
It is expected that  all  $ \X'  \in S_{\X}$ have the same  semantic meaning as $\X$ for human readers, but they may have different outputs from the model. 
The goal of the attacker is to find 
$ \X' \in S_{\X}$  such that 
$f(\X) \neq f( \X')$. 


\paragraph{Certified Robustness} 
Formally, 
a model $f$ 
 is said to be \emph{certified robust} against word substitution attacking on an input $\X$ if it is able to give consistently correct predictions for all the possible word substitution perturbations, i.e, 
 \begin{align}\label{equ:certified}
     \text{$y = f(\X) = f( \X')$,~~~~ for all $ \X' \in S_\X,$}
 \end{align} 
 where $y$ denotes the true label of sentence $\X$. 
Deciding if $f$ is certified robust can be highly challenging, 
because,
unless additional structural information is available, 
it requires to exam all the candidate sentences in $S_\X$, 
whose size grows exponentially with $R$.  In this work, we mainly consider the case when $R = L$, which is the most challenging case. 
\section{Certifying Smoothed Classifiers} 
Our idea is to replace $f$ with a more smoothed model that is easier to verify 
by averaging the outputs of a set of randomly perturbed inputs based on random word substitutions. 
The  smoothed classifier  $f^{\rs}$ is constructed by introducing random perturbations on the input space, 
\begin{align*} 
& f^{\rs}(\X)=\underset{c\in\mathcal{Y}}{\arg\max}\   
\mathbb{P}_{\Z\sim\Pi_{\X}}\left(f(\Z)=c\right), 
\end{align*}
%
where $\Pi_{\X}$ is a probability distribution on the input space that  prescribes a random perturbation around $\X$. 
For notation, we define 
$$
\grs(\X, c) :=\mathbb{P}_{\Z\sim\Pi_{\X}}\left(f(\Z)=c\right) ,
$$
which is the ``soft score'' of class $c$ under $f^\rs$. 

The perturbation distribution ${\Pi_\X}$ should be chosen properly so that $f^\rs$ forms a close approximation to the original model $f$ (i.e., 
$f^\rs(\X) \approx f(\X)$), and is also sufficiently random to ensure that $f^\rs$ is smooth enough to allow certified robustness (in the sense of Theorem~\ref{thm:lb} below).    
\cutforspace{ 
For vision tasks, Gaussian distribution is a natural choice for $\Pi_{\X}$ as the input is in a continuous Euclidean space  \cite{cohen2019certified}. 
However, 
Gaussian perturbation  can not be directly applied on 
discrete sentence inputs. 
}

In our work, 
we define $\Pi_\X$ to be the uniform distribution  
on a set of random word substitutions. 
Specifically, let $P_{\x}$ be 
a \emph{perturbation set} for word $\x$ in the vocabulary, 
which is different from the \emph{synonym set} $S_\x$. 
In this work, we construct $P_{\x}$ based on the top $K$ nearest neighbors 
under the cosine similarity of GLOVE vectors,
where $K$ is a hyperparameter that controls the size of the perturbation set; see Section \ref{sec:exp} for more discussion on $P_x$. 

For a sentence $\X=\x_{1},\ldots,\x_{L}$, 
the sentence-level perturbation distribution $\Pi_\X$ is 
defined by randomly and independently perturbing each  word $x_i$
to a word in its perturbation set $P_{x_i}$  
with equal probability, that is, 
\vspace{-10pt}
\[
\Pi_{\X}(\Z)=\prod_{i=1}^{L}\frac{\mathbb{I}\left\{ z_{i}\in P_{\x_{i}}\right\} }{\left|P_{\x_{i}}\right|},
\]
\vspace{-10pt}

\noindent where $\Z = z_{1},\ldots, z_{L}$ is the perturbed sentence  and $|P_{x_i}|$ denotes the size of $P_{x_i}$. 
Note that the random perturbation $\Z$ and the adversarial candidate $\X' \in S_\X$ are different. 
%
\begin{figure*}
    \vspace{-15pt}
    \centering
    \includegraphics[scale = 0.55]{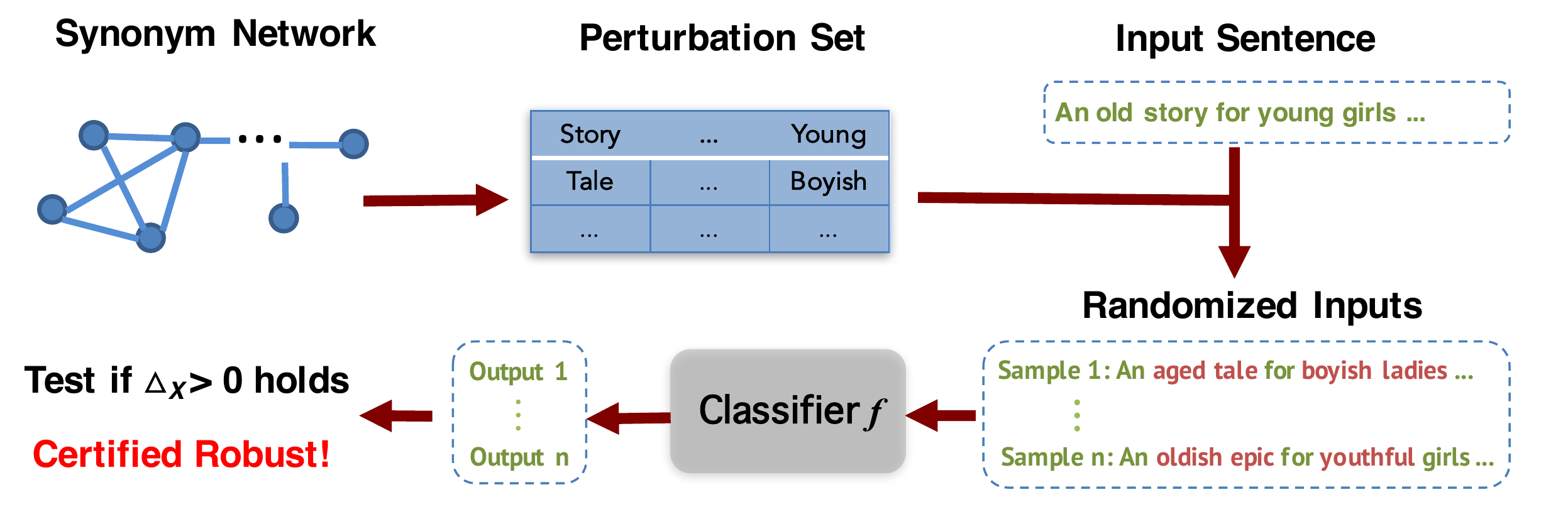}
    \vspace{-5pt}
    \caption{A pipeline of the proposed robustness certification approach.}
    \vspace{-15pt}
    \label{fig:framework}
\end{figure*}

\subsection{Certified Robustness}
We now discuss how to certify the robustness of the smoothed model $f^{\rs}$. 
Recall that $f^\rs$ is certified robust if $y  =  f^\rs(\X')$ for any $\X' \in S_{\X}$, where $y$ is the true label. 
A sufficient condition for this is 
$$
\min_{\X'\in S_{\X}} \grs(\X', y)  \geq 
\max_{\X' \in S_{\X}}\grs({\X'}, c)  ~~~~~\forall c \neq y, 
$$
where the lower bound of $\grs(\X', y)$ on $\X' \in S_\X$ 
is larger than the upper bound of $\grs(\X', c)$ on $\X' \in S_\X$ for every $c\neq y$. 
The key step is hence to calculate the upper and low bounds of $ \grs(\X', c) $ for $\forall c\in \mathcal Y$ and $\X'\in S_\X$, which we address in Theorem~\ref{thm:lb} below. 
All proofs are in Appendix \ref{apx:pf}.
\begin{theorem} \textbf{(Certified Lower/Upper Bounds)} \label{thm:lb}
Assume  the perturbation set $P_{\x}$ is constructed such that 
$\left|P_{\x}\right|=\left|P_{{\x'}}\right|$ for every word $x$ and its synonym ${\x'}\in S_{\x}$. Define 
\[
q_{\x}=\min_{{\x'}\in S_{\x}}{\left|P_{\x}\cap P_{{\x'}}\right|}/{\left|P_{\x}\right|},
\]
where $q_{\x}$ indicates the overlap between  the two different perturbation sets. 
For a given sentence $\X=\x_{1},\ldots,\x_{L}$, 
we sort the words according to $q_{\x}$, such that 
$q_{\x_{i_{1}}}\le q_{\x_{i_{2}}}\le\cdots\le q_{\x_{i_{L}}}$. 
Then
\begin{align*}
\min_{\X'\in S_{\X}}
\grs(\X',  c) 
& \ge\max(\grs(\X, c) -q_\X,0)\\
\max_{\X'\in S_{\X}}
\grs(\X',  c) 
& \le\min(\grs(\X, c) + q_\X,1). 
\end{align*}
where $q_\X:=1-\prod_{j=1}^{R}q_{\x_{i_{j}}}$. 
Equivalently, this says 
\begin{align*}
\left | \grs(\X',  c)  - \grs(\X, c) \right | \leq q_\X,~~~~
 \text{any label $c \in \mathcal Y$}. 
\end{align*}
\end{theorem}

The idea is that, with the randomized smoothing, 
the difference between $\grs(\X',  c)$  and $\grs(\X, c)$ is at most $q_\X$ for any adversarial candidate $\X' \in S_\X$. Therefore, we can give adversarial upper and lower bounds of $\grs(\X',  c)$ by $\grs(\X, c)\pm q_\X$, which,  
importantly, 
avoids the difficult adversarial optimization of $\grs(\X', c)$ on ${\X' \in S_\X}$, and instead just needs to evaluate $\grs(\X, c)$ at the original input $\X$. 


We are  ready to describe a practical criterion for checking the certified robustness. 
\begin{proposition} \label{pro:cAcB}
For a sentence $\X$ and its label $y$, we define 
\begin{align*} 
& \cB = \arg\max_{c\in \mathcal Y, c \neq y} \grs(\X, c). 
\end{align*}
Then under the condition of Theorem~\ref{thm:lb}, we can certify that $f(\X') = f(\X) = y$ for any $\X' \in S_{\X}$ if
\begin{align} \label{equ:DeltaX}
\Delta_\X \overset{def}{=} \grs(\X, y) - \grs(\X, \cB) - 2q_{\X} > 0. 
\end{align}
\end{proposition} 
Therefore, certifying whether the model gives consistently correct prediction reduces to checking if $\Delta_{
\X}$ is positive, which can be easily achieved with Monte Carlo estimation as we show in the sequel.  

\paragraph{Estimating $\grs(\X, c)$ and $\Delta_\X$} 
Recall that $\grs(\X, c)= \mathbb{P}_{\Z\sim\Pi_{\X}}(f(\Z)=c)$. 
We can estimate  $\grs(\X, c)$ with  a Monte Carlo estimator $\sum_{i=1}^{n}\mathbb{I} \{ f(\Z^{(i)})=c\}/n$, 
where $\Z^{(i)}$ are i.i.d. samples from $\Pi_{\X}$. 
And $\Delta_\X$ can be approximated accordingly. 
Using concentration inequality, we can quantify the non-asymptotic approximation error. 
This allows us to construct rigorous statistical procedures to  
reject the null hypothesis that $f^\rs$ is not certified robust at $\X$  (i.e., $\Delta_{\X} \leq 0$)  
with a given significance level (e.g., $1\%$).  
See Appendix \ref{apx:MC} for the algorithmic details of the testing procedure.   

We can see that our procedure is 
\emph{structure-free} in that it only requires the black-box assessment of the output $f(\Z^{(i)})$ of the random inputs, and does not require any other structural information of $f$ and $f^\rs$, which makes our method widely applicable to various types of complex models.  



\paragraph{Tightness} 
A key question is if our bounds are sufficiently tight. 
The next theorem shows that the  lower/upper bounds in  Theorem \ref{thm:lb} are  tight and can not be further improved unless  further information of the model $f$ or $f^\rs$ is acquired.  
\begin{theorem}
\textbf{(Tightness)} \label{thm:tight}
Assume the conditions of Theorem \ref{thm:lb} hold. 
For a model $f$ that satisfies $f^\rs(\X) = y$ and $\cB$ as defined in Proposition~\ref{pro:cAcB}, 
there exists a model $f_*$ such that its related smoothed classifier $\grsstar$ 
satisfies 
$\grsstar(\X, c)=\grs(\X, c)$ for $c = y$ and $c = \cB$, and 
\begin{align*}
\min_{\X'\in S_{\X}}
\grsstar(\X', y)  
& =\max(\grsstar(\X, y)-q_\X,0)\\
\max_{\X'\in S_{\X}}
\grsstar(\X',\cB) 
& =\min(\grsstar(\X, \cB) +q_\X,1),
\end{align*}
where $q_\X$ is defined in Theorem~\ref{thm:lb}. 
\end{theorem}
In other words, 
if we access  $\grs$ only through the evaluation of $\grs(\X, y)$ and $\grs(\X, \cB)$,
then the bounds in Theorem~\ref{thm:lb} are the tightest possible that we can achieve, because we can not distinguish between $\grs$ and the $\grsstar$ in Theorem~\ref{thm:tight} with the information available.  

\subsection{Practical Algorithm} \label{sec:prac}

Figure \ref{fig:framework} visualizes the pipeline of the 
proposed approach. Given the synonym sets $S_\X$, we generate the perturbation sets $P_\X$ from it. 
When an input sentence $\X$ arrives,  
we draw perturbed sentences $\{\Z^{(i)}\}$ from $\Pi_\X$ and average their outputs to estimate $\Delta_\X$, which is used to decide if the 
model is certified robust for $\X$. 

\paragraph{Training the Base Classifier $f$}
Our method needs to
start with a base classifier $f$. 
Although it is possible to train $f$ using standard learning techniques, the result can be improved by considering that the method uses the smoothed $f^\rs$, instead of $f$. 
To improve the accuracy of $f^\rs$, 
we introduce a data augmentation induced by the perturbation set. 
%
Specifically, at each training iteration, 
we first sample a mini-batch of data points (sentences) and randomly perturbing the sentences using the perturbation distribution $\Pi_\X$. 
We then apply gradient descent on the model based on the  perturbed mini-batch. 
Similar training procedures  
were also used for Gaussian-based random smoothing on continuous inputs \citep[see e.g.,][]{cohen2019certified}.

Our method can easily leverage powerful pre-trained models such as BERT. In this case, BERT is used to construct feature maps and only the top layer weights are finetuned using the data augmentation method. 

\section{Experiments} \label{sec:exp}
%
We test our method on both IMDB~\citep{maas2011imdb} and  Amazon~\citep{mcauley2013hidden}  text  classification  tasks, with  various types of models, including text CNN~\citep{kim2014textcnn}, Char-CNN~\citep{zhang2015charnn} and BERT~\citep{devlin2018bert}. 
We compare with  the recent IBP-based methods \citep{jia2019certified, huang2019achieving} as baselines.   
Text CNN ~\citep{kim2014textcnn} was used in \citet{jia2019certified} and achieves the best result therein. 
All the baseline models are trained and tuned using the schedules recommended in the corresponding papers. 
We consider the case when $R=L$ during attacking,  which means all words in the sentence can be perturbed simultaneously by the attacker. 
Code for reproducing our results can be found in \url{https://github.com/lushleaf/Structure-free-certified-NLP}.


\paragraph{Synonym Sets}
Similar to \citet{jia2019certified, alzantot2018generating}, we construct the synonym set $S_x$ of word $x$ to be the set of words with $\geq 0.8$  cosine similarity in the GLOVE vector space. 
The word vector space is constructed by post-processing the pre-trained GLOVE vectors~\citep{pennington2014glove}  using the counter-fitted method \citep{mrkvsic2016counter} and the ``all-but-the-top'' method \citep{mu2018allbutthetop} to ensure that synonyms are near to each other while antonyms are far apart.

\paragraph{Perturbation Sets}
We say that two words $x$ and $ x'$ are connected synonymously  if there exists a path of words $x = x_1, x_2, \ldots, x_\ell = x'$, such that all the successive pairs are synonymous. 
Let $B_x$ to be the set of words connected to $x$ synonymously. 
Then we  define the perturbation set $P_x$ to consist of the top $K$ words in $B_x$ with the largest GLOVE cosine similarity if $|B_x| \geq K$, and set $P_x = B_x$ if $|B_x| < K$. 
Here $K$ is a hyper-parameter that controls the size of $P_x$ and hence 
trades off the smoothness and accuracy of $f^\rs$. 
We use $K=100$ by default and investigate its effect in 
Section \ref{apx:trades}.

\paragraph{Evaluation Metric}
We evaluate the certified robustness of a model $f^\rs$ 
on a dataset with the \emph{certified accuracy} \citep{cohen2019certified}, 
which equals the percentage of data points on which 
$f^\rs$ is certified robust, which, for our method, holds when $\Delta_\X > 0$ can be verified. 
\cutforspace{ 
For the IBP-based methods \citep{jia2019certified, huang2019achieving},
the certified accuracy is calculated similarly, with certified robustness holds 
when  the lower bound of the logits of the true labels are larger than the upper bounds of all the other classes based on IBP.  
}

\subsection{Main Results} 
We first demonstrate that adversarial word substitution is able to give strong attack in our experimental setting. Using IMDB dataset, we attack the vanilla BERT \citep{devlin2018bert} with the adversarial attacking method of \citet{jin2019bert}. The vanilla BERT achieves a $91\%$ clean accuracy (the testing accuracy on clean data without attacking), but only a $20.1\%$ adversarial accuracy (the testing accuracy under the particular attacking method by \citet{jin2019bert}). We will show later that our method is able to achieve $87.35\%$ certified accuracy and thus the corresponding adversarial accuracy must be higher or equal to $87.35\%$.

We  compare our method with IBP ~\citep{jia2019certified, huang2019achieving}.
in Table~\ref{tab:textcnn}. 
We can see that our method clearly outperforms the baselines. 
In particular, 
our approach  significantly outperforms IBP on Amazon by improving the $14.00\%$ baseline to $24.92\%$.

Thanks to its structure-free property, 
our algorithm can be easily applied to any pre-trained models and character-level models, 
which is not easily achievable with  \citet{jia2019certified} and \citet{huang2019achieving}.  
%
%
Table ~\ref{tab:model} shows that 
our method can further improve the result using Char-CNN (a character-level model) and 
BERT~\citep{devlin2018bert},  
achieving an $87.35\%$ certified accuracy on IMDB. 
In comparison, the IBP baseline only achieves a $79.74\%$ certified  accuracy under the same setting. 

\begin{table}[btp]
\begin{center}\scalebox{.9}{ 
\begin{tabular}{l|cc}
\toprule
\bf Method & \bf IMDB & \bf Amazon  \\ 
\hline
\citet{jia2019certified} & 79.74 & 14.00 \\
\citet{huang2019achieving} & 78.74 & 12.36 \\
Ours & \bf 81.16 & \bf 24.92 \\
\bottomrule
\end{tabular}}
\end{center}
\caption{\label{tab:textcnn} The certified accuracy of our method and the baselines on the IMDB and Amazon dataset.} 
\end{table}

\begin{table}[btp]
\begin{center}\scalebox{.9}{ 
\begin{tabular}{l|l|c}
\toprule
\bf Method & \bf Model & \bf Accuracy  \\ 
\hline
\citet{jia2019certified} & CNN & 79.74 \\
\citet{huang2019achieving} & CNN & 78.74 \\
\hline
\multirow{3}{*}{Ours} & CNN & 81.16 \\
& Char-CNN & 82.03 \\
& BERT & \bf 87.35 \\
\bottomrule
\end{tabular}}
\end{center}
\caption{\label{tab:model} The certified accuracy of different models and methods on the IMDB dataset.}
\end{table}


\subsection{Trade-Off between Clean Accuracy and Certified Accuracy}
\label{apx:trades}
We 
investigate the trade-off between smoothness and accuracy while tuning $K$ in Table \ref{tab:trades}.
We can see that the clean accuracy decreases when $K$ increases, while 
the gap between the clean accuracy and certified accuracy,
which measures the smoothness, decreases when $K$ increases. The best certified accuracy is achieved when $K=100$. 
%
\begin{table}[hbtp]
\begin{centering}
\begin{tabular}{l|ccccc}
\toprule
$K$ & 20 & 50 & 100 & 250 & 1000 \\
\hline 

\small{Clean (\%)} & \small{88.47} & \small{88.48} & \small{88.09} & \small{84.83} & \small{67.54} \\
\small{Certified (\%)} & \small{65.58} & \small{77.32} & \small{81.16} & \small{79.98} & \small{65.13} \\
\bottomrule
\end{tabular}\caption{\label{tab:trades}Results of the smoothed model $f^\rs$ with  different $K$ on IMDB using text CNN. ``Clean" represents the accuracy on the clean data without adversarial attacking and ``Certified" the certified accuracy. 
}
\par\end{centering}
\end{table}

\section{Conclusion}
We proposed a robustness certification method, which provably guarantees that all the possible perturbations cannot break down the system. Compared with previous work such as  \citet{jia2019certified, huang2019achieving}, our method is structure-free and thus can be easily applied to any pre-trained models (such as BERT) and character-level models (such as Char-CNN). 

The construction of the perturbation set is of critical importance to our method. In this paper, we used a heuristic way based on the synonym  network to construct the perturbation set, which may  not be optimal. In further work, we will explore more efficient ways for constructing the perturbation set. We also plan to generalize our approach to achieve certified robustness against other types of adversarial attacks in NLP,  such as the out-of-list attack. An na\"ive way is to add the ``OOV" token  into the synonyms set of  every word, but potentially better procedures can be further explored. 

\section*{Acknowledgement}
This work is supported in part by NSF CRII 1830161 and NSF CAREER 1846421.

\newpage
\bibliography{acl2020}
\bibliographystyle{acl_natbib}

\newpage
\onecolumn
\appendix
\section{Appendix}

\subsection{Bounding the Error of Monte Carlo Estimation} \label{apx:MC}
As shown in Proposition~\ref{pro:cAcB}, 
the smoothed model $f^\rs$ is certified robust  at an input $\X$ in the sense of \eqref{equ:certified}
if 
\begin{align*} 
\Delta_\X 
& = \grs(\X, y)  
-\grs(\X, \cB) 
-2 q_\X \\
& = \grs(\X, y)  
- \max_{c\ne y}
\grs(\X, c) 
-2q_\X > 0, 
\end{align*}
where $y$ is the true label of $\X$,   
and 
$$
\grs(\X, c) := \mathbb{P}_{\Z\sim\Pi_{\X}}\left(f(\Z)=c\right) 
= 
\mathbb{E}_{\Z\sim\Pi_{\X}}\left[\mathbb I\{ f(\Z)=c\}\right]. 
$$
Assume $\{\Z^{(i)}\}_{i=1}^n$ is an i.i.d. sample 
from $\Pi_{\X}$. 
By Monte Carlo approximation, we can estimate $\grs(\X, c)$ for all $c\in\mathcal{Y}$ jointly, via
 $$
 \hatgrs(\X, c ):=\frac{1}{n}\sum_{i=1}^{n}\mathbb{I}\left\{ f(\Z^{(i)})=c\right\},$$
 and estimate $\Delta_\X$ via 
$$
\hat \Delta_\X := 
\frac{1}{n}\sum_{i=1}^{n}\mathbb{I}\left\{ f(\Z^{(i)})=y\right\}  - 
\max_{c\ne y}\frac{1}{n}\sum_{i=1}^{n}\mathbb{I}\left\{ f(\Z^{(i)})=c\right\} 
- 2 q_\X. 
$$
To develop a rigorous procedure for testing $\Delta_\X >0$, 
we need to bound the non-asymptotic error of the Monte
Carlo estimation, which can be done with a simple application of Hoeffding's concentration inequality and union bound. 
\begin{proposition}\label{pro:concentration}
Assume $\{\Z^{(i)}\}$ is i.i.d. drawn from $\Pi_\X$. 
For any $\delta\in (0,1)$, with probability at least $1-\delta$,
we have 
\begin{align*}
\Delta_\X \ge & 
\hat\Delta_\X - 2\sqrt{\frac{\log\frac{1}{\delta}+\log\left|\mathcal{Y}\right|}{2n}}.
\end{align*}
\end{proposition}
%
We can now frame the robustness certification problem into a hypothesis test problem.
Consider the null hypothesis $\text{H}_0$ and alternatively hypothesis $\text{H}_a$:
\begin{align*}
\text{H}_{0}:  &  \Delta_\X \leq 0~~~ (f^{\rs}\text{\ is\ not\ certified\ robust\ to\ }\X)\\
\text{H}_{a}: &  \Delta_\X > 0~~~ (f^{\rs}\text{\ is\ certified\ robust\ to\ }\X).
\end{align*}
Then according to Proposition~\ref{pro:concentration}, we can reject the null hypothesis $\text{H}_0$ with a significance level $\delta$ if 
\[
\hat\Delta_\X - 2\sqrt{\frac{\log\frac{1}{\delta}+\log\left|\mathcal{Y}\right|}{2n}} > 0. 
\]
In all the experiments, we set $\delta = 0.01$ and $n=5000$.
\subsection{Proof of the Main Theorems} \label{apx:pf}
In this section, we give the proofs of the theorems in the main text.

\newcommand{\Xconstraint}{{\X'\in S_\X}}
\renewcommand{\H}{\mathcal{H}_{[0,1]}}

\subsubsection{Proof of Proposition \ref{pro:cAcB}}
According to the definition of $f^\rs$, it is certified robust at $\X$, that is, $y = f^\rs(\X')$ for $\forall \X' \in S_\X$,  if 
\begin{align} \label{equ:grddfdi}
\grs(\X', ~ y) \geq \max_{c\neq y}\grs(\X', c),~~~~~  \X' \in S_\X. 
\end{align}
Obviously 
\begin{align*} 
\grs(\X', ~ y) - \max_{c\neq y}\grs(\X', c) 
& \geq \min_{\X'\in S_\X}\grs(\X', ~ y) - \max_{c\neq y} \max_{\X' \in S_\X}\grs(\X', c)  \\
& \geq \left (\grs(\X, ~ y) - q_\X \right) - \max_{c\neq y} \left (\grs(\X, c) + q_\X \right) 
~~~\text{//by Theorem~\ref{thm:lb}.} \\
& = \Delta_\X. 
\end{align*}
Therefore, $\Delta_\X > 0$ must imply \eqref{equ:grddfdi} and hence certified robustness. 

\subsubsection{Proof of Theorem \ref{thm:lb}}

Our goal is to calculate the upper and lower bounds 
$\max_{\X'\sim \Pi_\X}\grs(\X', c)$ and 
$\min_{\X'\sim \Pi_\X}\grs(\X', c)$.  
Our key idea is to frame the computation of the upper and lower bounds into a variational optimization. 
\begin{lemma}\label{lem:first}
Define $\H$ to be the set of all bounded functions mapping from $\mathcal{X}$ to $[0,1]$,  For any $h \in \H$, define 
\[
\Pi_{\X}[h]=\mathbb{E}_{\Z\sim\Pi_{\X}}[h(\Z)].
\]
Then we have for any $\X$ and $c \in \mathcal Y$,  
\begin{align*} 
& \min_{\X'\sim \Pi_\X}\grs(\X', c) \geq  \min_{h \in \H }\min_{\X'\sim \Pi_\X} \left \{ \Pi_{\X'}[h]  ~~~s.t.~~~  \Pi_{\X}[h]  = \grs(\X, c) \right\} : =  g^{\rs}_{low}(\X, c), 
\\
& \max_{\X'\sim \Pi_\X}\grs(\X', c) \leq  \max_{h \in \H }\max_{\X'\sim \Pi_\X} \left \{ \Pi_{\X'}[h]  ~~~s.t.~~~  \Pi_{\X}[h]  = \grs(\X, c) \right\} : = g^{\rs}_{up}(\X, c). 
\end{align*}
\end{lemma}
\begin{proof}[Proof of Lemma~\ref{lem:first}]
The proof is straightforward. Define $h_0(\X)= \mathbb I\{ f(\X)=c\}$. 
Recall that 
$$
\grs(\X, c) = \mathbb{P}_{\Z\sim\Pi_{\X}}\left(f(\Z)=c\right) = \Pi_{\X}[h_0]. 
$$
Therefore, $h_0$ satisfies the constraints in the optimization, which makes it obvious that 
$$
\grs(\X', c)  = \Pi_{\X'}[h_0] \geq 
  \min_{h \in \H }\left \{ \Pi_{\X'}[h]  ~~~s.t.~~~  \Pi_{\X}[h]  = \grs(\X, c) \right\}.  
$$
Taking $\min_{\X' \in S_\X}$ on both sides yields the lower bound. The upper bound follows the same derivation. 
\end{proof}

Therefore, the problem reduces to deriving bounds for the optimization problems. 
\begin{theorem} \label{thm:constrain}
Under the assumptions of Theorem~\ref{thm:lb}, 
for the optimization problems in Lemma~\ref{lem:first}, 
we have 
\begin{align*} 
 g^{\rs}_{low}(\X, c) \geq \max (\grs(\X,c) - q_\X,~ 0), 
&& 
g^{\rs}_{up}(\X, c) \leq \min(\grs(\X, c) + q_\X, ~ 1). 
\end{align*}
where $q_\X$ is the quantity defined in Theorem \ref{thm:lb} in the main text. 
\end{theorem}
Now we proceed to prove Theorem \ref{thm:constrain}. 

\begin{proof}[Proof of Theorem~\ref{thm:constrain}]
We only consider the minimization problem because the maximization follows the same proof. For notation, we denote $p = \grs(\X, c)$. 
Applying the Lagrange multiplier to the constraint optimization problem
and exchanging the min and max, we have 
\begin{align*}
g^{\rs}_{low}(\X, c) = &  \min_{\Xconstraint}\min_{h\in\H}\mathrm{\max_{\lambda\in \mathbb{R}}}\ \Pi_{\X'}[h]-\lambda\Pi_{\X}[h]+\lambda p\\
\ge & \mathrm{\max_{\lambda\in \mathbb{R}}}\min_{\Xconstraint}\min_{h\in\H}\ \Pi_{\X'}[h]-\lambda\Pi_{\X}[h]+\lambda p\\
= & \mathrm{\max_{\lambda\in \mathbb{R}}}\min_{\Xconstraint}\min_{h\in\H}\ \int h(\Z)\left(d\Pi_{\X'}(\Z)-\lambda d\Pi_{\X}(\Z)\right)+\lambda p\\
= & -\mathrm{\max_{\lambda\in \mathbb{R}}}\max_{\Xconstraint}\int\left(\lambda d\Pi_{\X}(\Z)-d\Pi_{\X'}(\Z)\right)_{+}+\lambda p
\\
= & -\mathrm{\max_{\lambda\ge0}}\max_{\Xconstraint}\int\left(\lambda d\Pi_{\X}(\Z)-d\Pi_{\X'}(\Z)\right)_{+}+\lambda p.
\end{align*}
Here $d\Pi_{\X}^{0}(\Z)$ and $d\Pi_{\X'}^{0}(\Z)$ is the counting
measure and $(s)_{+}=\max(s,0)$. Now we calculate $\int\left(\lambda d\Pi_{\X}(\Z)-d\Pi_{\X'}(\Z)\right)_{+}$.

\begin{lemma} \label{lem:analytic}
Given $\x,\x'$, define $n_{\x}=\left|P_{\x}\right|$, $n_{\x'}=\left|P_{\x'}\right|$
and $n_{\x,\x'}=\left|P_{\x}\cap P_{\x'}\right|$. We have the following identity
\begin{align*}
 & \int\left(\lambda d\Pi_{\X}(\Z)-d\Pi_{\X'}(\Z)\right)_{+} 
 \\
 = & \lambda\left[1-\prod_{j\in[L],\x_{j}\neq\x'_{j}}\frac{n_{\x_{j},\x_{j}'}}{n_{\x_{j}}}\right]+\left[\prod_{j\in[L],\x_{j}\neq\x'_{j}}\frac{n_{\x_{j},\x_{j}'}}{n_{\x_{j}}}\right]\left[\lambda-\prod_{j\in[L],\x_{j}\ne\x'_{j}}\frac{n_{\x_{j}}}{n_{\x'_{j}}}\right]_{+}.
\end{align*}
As a result, 
under the assumption that $n_x=\left|P_{\x}\right|=\left|P_{\x'}\right|=n_{x'}$ for every word $x$  and 
its synonym $\x'\in S_{\x}$,  we have 
\begin{align*}
\int  \left(\lambda d\Pi_{\X}(\Z)-d\Pi_{\X'}(\Z)\right)_{+} 
= & 
\lambda\left[1-\prod_{j\in[L],\x_{j}\neq\x'_{j}} 
 \frac{n_{\x_{j},\x_{j}'}}{n_{\x_{j}}}
\right]+\left[\prod_{j\in[L],\x_{j}\neq\x'_{j}} 
\frac{n_{\x_{j},\x_{j}'}}{n_{\x_{j}}}
\right]\left(\lambda-1\right)_{+}.
\end{align*}

\end{lemma}
We now need to solve the optimization of  $\max_{\Xconstraint}\int\left(\lambda d\Pi_{\X}(\Z)-d\Pi_{\X'}(\Z)\right)_{+}$. 

\newcommand{\txstar}{\tilde{\x}^*}
\begin{lemma} \label{lem:optperterb}
For any word $\x$, define $\txstar=\underset{\x'\in S_{\x}}{\arg\min}\ n_{\x,\x'}/n_{\x}$.
For a given sentence $\X = x_1,\ldots, x_L$, we define an ordering of the words 
$x_{\ell_1}, \ldots, x_{\ell_L}$ such that  
 $n_{\x_{\ell_{i}},\txstar_{\ell_{i}}}/n_{\x_{\ell_{i}}}\le n_{\x_{\ell_{j}},\txstar_{\ell_{j}}}/n_{\x_{\ell_{j}}}$
for any $i\le j$.  
For  a given $\X$ and $R$, we define an adversarial perturbed sentence 
$
\X^* =  x^*_1,\ldots,  x^*_L, $ 
where
$$
 x_i^*  = \begin{cases}
\tilde{\x}_{i}^{*}  & \text{if $i \in [\ell_1, \ldots, \ell_R]$} \\
{\x}_{i}  & \text{if $i \notin [\ell_1, \ldots, \ell_R]$}. \\
\end{cases}
$$
Then  for any $\lambda\ge0$, we have that $\X^*$ is the optimal solution of  $\max_{\Xconstraint}\int\left(\lambda d\Pi_{\X}(\Z)-d\Pi_{\X'}(\Z)\right)_{+}$, that is, 
\[
\max_{\Xconstraint}\int\left(\lambda d\Pi_{\X}(\Z)-d\Pi_{\X'}(\Z)\right)_{+}=\int\left(\lambda d\Pi_{\X}(\Z)-d\Pi_{\X^*}(\Z)\right)_{+}.
\]
\end{lemma}

Now by Lemma \ref{lem:optperterb}, the lower bound becomes
\begin{align}
g^{\rs}_{low}(\X, c) 
= ~& -\mathrm{\max_{\lambda\ge0}}\max_{\Xconstraint}\int\left(\lambda d\Pi_{\X}(\Z)-d\Pi_{\X'}(\Z)\right)_{+}+\lambda p \notag \\
 =~ &   - \mathrm{\max_{\lambda\ge0}}\int\left(\lambda d\Pi_{\X}(\Z)-d\Pi_{\X^*}(\Z)\right)_{+}+\lambda p\notag \\
= ~&  \mathrm{\max_{\lambda\ge0}}\mathrm{\ }(p-q_\X)\lambda-(1-q_\X)(\lambda-1)_{+} \label{equ:dfdfdjijiojioj}\\
= ~ & \max( p - q_{\X}, ~ 0 ), \notag
\end{align}
where $q_\X$ is consistent with the definition in Theorem~\ref{thm:lb}: \[
q_\X=1-\prod_{j\in[L],\x_{j}\neq\tilde{\x}_{j}^{*}}\frac{n_{\x_{j},\tilde{\x}_{j}^{*}}}{n_{\x_{j}}} = 
1- \prod_{j=1}^R q_{x_{\ell_j}}.
\]
Here equation (\ref{equ:dfdfdjijiojioj}) is by calculation using the assumption of Theorem \ref{thm:lb}.
The optimization of $\max_{\lambda\geq 0}$ in \eqref{equ:dfdfdjijiojioj} is an elementary step: 
if $p\le q$, we have $\lambda^{*}=0$ with
solution $0$; if $p\ge q$, we have $\lambda^{*}=1$ with solution $(p-q_\X)$.
This finishes the proof of the lower bound. The proof the upper bound follows similarly. 
\end{proof}

\paragraph{Proof of Lemma \ref{lem:analytic}}

Notice that we have
\begin{align*}
  \int\left(\lambda d\Pi_{\X}(\Z)-d\Pi_{\X'}(\Z)\right)_{+}
= & \sum_{\Z\in S_{\X'}\cap S_{\X}}\left(\lambda\left|S_{\X}\right|^{-1}-\left|S_{\X'}\right|^{-1}\right)_{+}
+  \lambda\sum_{\Z\in S_{\X}-S_{\X'}}\left|S_{\X}\right|^{-1}\\
= & \left|S_{\X'}\cap S_{\X}\right|\left(\lambda\left|S_{\X}\right|^{-1}-\left|S_{\X'}\right|^{-1}\right)_{+}
+  \lambda\left|S_{\X}-S_{\X'}\right|\left|S_{\X}\right|^{-1}.
\end{align*}
Also notice that $\left|S_{\X}\right| =\prod_{j=1}^{L}n_{\x_{j}}$; $\left|S_{\X'}\right| =\prod_{j=1}^{L}n_{\x_{j}'}$; $\left|S_{\X'}\cap S_{\X}\right| = \prod_{j=1}^{L}n_{\x_{j},\x_{j}'}$ and $\left|S_{\X}-S_{\X'}\right| =\prod_{j=1}^{L}n_{\x_{j}}-\prod_{j=1}^{L}n_{\x_{j},\x_{j}'}$.
Plugging in the above value, we have 
\begin{align*}
  \left|S_{\X}-S_{\X'}\right|\left|S_{\X}\right|^{-1} 
= & \frac{\prod_{j=1}^{L}n_{\x_{j}}-\prod_{j=1}^{L}n_{\x_{j},\x_{j}'}}{\prod_{j=1}^{L}n_{\x_{j}}}\\
= & 1-\prod_{j=1}^{L}\frac{n_{\x_{j},\x_{j}'}}{n_{\x_{j}}}\\
= & 1-\prod_{j\in[L],\x_{j}\neq\x'_{j}}\frac{n_{\x_{j},\x_{j}'}}{n_{\x_{j}}}.
\end{align*}
And also,
\begin{align*}
  \left(\lambda\left|S_{\X}\right|^{-1}-\left|S_{\X'}\right|^{-1}\right)_{+}
= & \left(\lambda\prod_{j=1}^{L}n_{\x_{j}}^{-1}-\prod_{j=1}^{L}n_{\x_{j}'}^{-1}\right)_{+}\\
= & \left(\lambda\prod_{j\in[L],\x_{j}=\x'_{j}}n_{\x_{j}}^{-1}\prod_{j\in[L],\x_{j}\ne\x'_{j}}n_{\x_{j}}^{-1}-\prod_{j\in[L],\x_{j}=\x'_{j}}n_{\x_{j}}^{-1}\prod_{j\in[L],\x_{j}\ne\x'_{j}}n_{\x'_{j}}^{-1}\right)_{+}\\
= & \prod_{j\in[L],\x_{j}=\x'_{j}}n_{\x_{j}}^{-1}\left(\lambda\prod_{j\in[L],\x_{j}\ne\x'_{j}}n_{\x_{j}}^{-1}-\prod_{j\in[L],\x_{j}\ne\x'_{j}}n_{\x'_{j}}^{-1}\right)_{+}.
\end{align*}
Plugging in the above value, we have 
\begin{align*}
  \left|S_{\X'}\cap S_{\X}\right|\left(\lambda\left|S_{\X}\right|^{-1}-\left|S_{\X'}\right|^{-1}\right)_{+}
= & \prod_{j=1}^{L}n_{\x_{j},\x_{j}'}\left(\lambda\left|S_{\X}\right|^{-1}-\left|S_{\X'}\right|^{-1}\right)_{+}\\
= & \prod_{j\in[L],\x_{j}=\x'_{j}}n_{\x_{j}}\prod_{j\in[L],\x_{j}\neq\x'_{j}}n_{\x_{j},\x_{j}'}\left(\lambda\left|S_{\X}\right|^{-1}-\left|S_{\X'}\right|^{-1}\right)_{+}\\
= & \prod_{j\in[L],\x_{j}\neq\x'_{j}}n_{\x_{j},\x_{j}'}\left(\lambda\prod_{j\in[L],\x_{j}\ne\x'_{j}}n_{\x_{j}}^{-1}-\prod_{j\in[L],\x_{j}\ne\x'_{j}}n_{\x'_{j}}^{-1}\right)_{+}\\
= & \prod_{j\in[L],\x_{j}\neq\x'_{j}}n_{\x_{j},\x_{j}'}\prod_{j\in[L],\x_{j}\ne\x'_{j}}n_{\x_{j}}^{-1}\left(\lambda-\prod_{j\in[L],\x_{j}\ne\x'_{j}}\frac{n_{\x_{j}}}{n_{\x'_{j}}}\right)_{+}\\
= & \left(\prod_{j\in[L],\x_{j}\neq\x'_{j}}\frac{n_{\x_{j},\x_{j}'}}{n_{\x_{j}}}\right)\left(\lambda-\prod_{j\in[L],\x_{j}\ne\x'_{j}}\frac{n_{\x_{j}}}{n_{\x'_{j}}}\right)_{+}.
\end{align*}
Combining all the calculation, we  get 
\begin{align*}
&  \int\left(\lambda d\Pi_{\X}(\Z)-d\Pi_{\X'}(\Z)\right)_{+} 
\\
= & \lambda\left[1-\prod_{j\in[L],\x_{j}\neq\x'_{j}}\frac{n_{\x_{j},\x_{j}'}}{n_{\x_{j}}}\right]+\left[\prod_{j\in[L],\x_{j}\neq\x'_{j}}\frac{n_{\x_{j},\x_{j}'}}{n_{\x_{j}}}\right]\left[\lambda-\prod_{j\in[L],\x_{j}\ne\x'_{j}}\frac{n_{\x_{j}}}{n_{\x'_{j}}}\right]_{+}.
\end{align*}

\paragraph{Proof of Lemma \ref{lem:optperterb}}

It is sufficient to proof that, for any $\X'\neq\X^*$, we have
\[
\int\left(\lambda d\Pi_{\X}(\Z)-d\Pi_{\X^*}(\Z)\right)_{+}\ge\int\left(\lambda d\Pi_{\X}(\Z)-d\Pi_{\X'}(\Z)\right)_{+}.
\]
Notice that for any $\lambda\ge0$, define
\[
Q(\X,\X'')=\lambda\left[1-\prod_{j\in[L],\x_{j}\neq\x'_{j}}\frac{n_{\x_{j},\x_{j}''}}{n_{\x_{j}}}\right]+\left[\prod_{j\in[L],\x_{j}\neq\x'_{j}}\frac{n_{\x_{j},\x_{j}''}}{n_{\x_{j}}}\right]\left(\lambda-1\right)_{+}.
\]
Given any \X, we can view $Q(\X,\X'')$ as the function of $n_{\x_{i},\x_{i}''}/n_{\x_{i}}$, $i\in[L]$. And $Q(\X,\X'')$ is a decreasing function of $n_{\x_{i},\x_{i}''}/n_{\x_{i}}$
for any $i\in[L]$ when fixing $\frac{n_{\x_{j},\x_{j}''}}{n_{\x_{j}}}$ for all other $j \neq i$. Suppose $\tilde{r}_k$ is the $k$-th smallest quantities of $n_{\x_{i},\tilde{\x}_{i}^{*}}/n_{\x_{i}}$, $i\in[L]$ and $r'_k$ is the $k$-th smallest quantities of $n_{\x_{j},\tilde{\x}_{j}^{*}}/n_{\x_{i}}$, $i\in[L]$. By the construction of $\X^*$, we have $\tilde{r}_k \le r'_k$ for any $k\in[L]$. This implies that 
\[
Q(\X,\X^*)\ge Q(\X,\X').
\]




\subsubsection{Proof of Theorem \ref{thm:tight}}
We denote $\grs(\X, y) = p_A$, $\grs(\X, \cB) = p_B$ and $q = q_\X$ in this proof for simplicity.
The $\X^*$ below is the one defined in the proof of Lemme \ref{lem:optperterb}.
Our proof is based on constructing a randomized smoothing classifier that
satisfies the desired property we want to prove.

\paragraph{Case 1 $p_{A}\ge q$ and $p_{B}+q\le1$}
Note that in this case $\left|S_{\X}\cap S_{\X^{*}}\right|/\left|S_{\X}\right|=1-q\ge(p_{A}-q)+p_{B}$,
where the inequality is due to $p_{A}+p_{B}\le1$. Therefore, we can
choose set $U_{1}$ and $U_{2}$ such that $U_{1}\subseteq S_{\X}\cap S_{\X^{*}}$;
$U_{2}\subseteq S_{\X}\cap S_{\X^{*}}$; $U_{1}\cap U_{2}=\emptyset$;
$\left|U_{1}\right|/\left|S_{\X}\right|=p_{A}-q$ and $\left|U_{2}\right|/\left|S_{\X}\right|=p_{B}$. We
define the classifier:
\[
f^{*}(\Z)=\begin{cases}
y & \text{if}\ \Z\in\left(S_{\X}-S_{\X^{*}}\right)\cap U_{1}\\
\cB & \text{if}\ \Z\in \left(S_{\X^{*}}-S_{\X}\right) \cup U_{2} \\
\text{other class ($c \neq y$ or $\cB$)} & \text{if}\ \Z\in S_{\X}\cap S_{\X^{*}}-\left(U_{1}\cup U_{2}\right)\\
\text{any\ class ($c\in \mathcal Y$)} & \text{otherwise}
\end{cases}
\]
This classifier is well defined for binary classification because
$S_{\X}\cap S_{\X^{*}}-\left(U_{1}\cup U_{2}\right)=\emptyset$.

\paragraph{Case 2 $p_{A}<q$ and $p_{B}+q\le1$}

In this case, we can choose set $U_{1}$ and $U_{2}$ such that $U_{1}\subseteq S_{\X}-S_{\X^{*}}$;
$U_{2}\subseteq S_{\X}\cap S_{\X^{*}}$; $\left|U_{1}\right|/\left|S_{\X}\right|=p_{A}$
and $\left|U_{2}\right|/\left|S_{\X}\right|=p_{B}$. We define the classifier: 
\[
f^{*}(\Z)=\begin{cases}
y & \text{if}\ \Z\in U_{1}\\
\cB & \text{if}\ \Z\in U_{2}\cup\left(S_{\X^{*}}-S_{\X}\right)\\
\text{other class ($c \neq y$ or $\cB$)} & \text{if}\ \Z\in S_{\X}-\left(U_{1}\cup U_{2}\right)\\
\text{any\ class ($c\in \mathcal Y$)} & \text{otherwise}
\end{cases}
\]
This classifier is well defined for binary classification because
$S_{\X}-\left(U_{1}\cup U_{2}\right)=\emptyset$.

\paragraph{Case 3 $p_{A}\ge q$ and $p_{B}+q>1$}

This case does not exist since we would have $p_{A}+p_{B}>1$.

\paragraph{Case 4 $p_{A}<q$ and $p_{B}+q>1$}

We choose set $U_{1}$ and $U_{2}$ such that $U_{1}\subseteq S_{\X}-S_{\X^{*}}$; 
$U_{2}\in S_{\X}-S_{\X^{*}}$; $U_{1}\cap U_{2}=\emptyset$; $\left|U_{1}\right|/\left|S_{\X}\right|=p_{A}$
and $\left|U_{2}\right|/\left|S_{\X}\right|=p_{B}-(1-q)$. Notice that the intersect of
$U_{1}$ and $U_{2}$ can be empty as $\left|U_{1}\right|/\left|S_{\X}\right|+\left|U_{2}\right|/\left|S_{\X}\right|=p_{A}+p_{B}-(1-q)\le1-(1-q)=q=\left|S_{\X}-S_{\X^{*}}\right|/\left|S_{\X}\right|$.
We define the classifier:

\[
f^{*}(\Z)=\begin{cases}
y & \text{if}\ \Z\in U_{1}\\
\cB & \text{if}\ \Z\in U_{2}\cup S_{\X^{*}}\\
\text{other class ($c \neq y$ or $\cB$)} & \text{if}\ \Z\in\left(S_{\X}-S_{\X^{*}}\right)-\left(U_{1}\cup U_{2}\right)\\
\text{any\ class ($c\in \mathcal Y$)} & \text{otherwise}
\end{cases}
\]
This classifier is well defined for binary classification because
$S_{\X}-S_{\X^{*}}-\left(U_{1}\cup U_{2}\right)=\emptyset$.

It can be easily verified that for each case, the defined classifier
satisfies all the conditions in Theorem \ref{thm:tight}.

\section{Additional Experiment Details}

We set $R = L$ in adversarial attacking, that is, 
all  words in the sentence can be perturbed simultaneously by the attacker. 
We use 5,000 random draws in the Monte Carlo estimation of $\Delta_\X$, and use the same method in \citet{jia2019certified} to tune the hyper-parameters when training the base models {e.g. learning rate, batch size and the schedule of loss function}. 
For the IMDB dataset, we train the IBP models and ours for 60 and 10 epochs, respectively. For the Amazon dataset, we train the IBP models and ours for 100 and 20 epochs, respectively.

We test our algorithm on two different datasets, IMDB and Amazon.
The IMDB movie review dataset~\citep{maas2011imdb} is a sentiment classification dataset. 
It consists of 50,000 movie review comments with binary sentiment labels.
The Amazon review dataset~\citep{mcauley2013hidden} is an extremely large dataset
that contains 34,686,770 reviews with 5 different types of labels. Similar to \citet{cohen2019certified}, we test the models on randomly selected subsets of the  test set with 1,250 and 6,500 examples for IMDB and Amazon dataset, respectively.

\end{document}